\documentclass{ws-ijufks} 
\usepackage[super]{cite}
\usepackage[utf8]{inputenc} % allow utf-8 input
\usepackage[T1]{fontenc}    % use 8-bit T1 fonts
\usepackage{hyperref}       % hyperlinks
\usepackage{url}            % simple URL typesetting
\usepackage{booktabs}       % professional-quality tables
\usepackage{amsfonts}       % blackboard math symbols
\usepackage{nicefrac}       % compact symbols for 1/2, etc.
\usepackage{microtype}      % microtypography
\usepackage[table, dvipsnames]{xcolor}
\usepackage{graphicx}
\usepackage{amssymb}
\expandafter\let\csname equation*\endcsname\relax
\expandafter\let\csname endequation*\endcsname\relax
\usepackage{amsmath}
\usepackage{multirow}
\usepackage{bbm}
\usepackage{bm} 

\newcommand{\beginsupplement}{%
	\setcounter{section}{0}
	\renewcommand{\thesubsection}{\Alph{subsection}}
        \setcounter{table}{0}
        \renewcommand{\thetable}{\Alph{subsection}\arabic{table}}%
        \setcounter{figure}{0}
        \renewcommand{\thefigure}{\Alph{subsection}\arabic{figure}}%
        }
\DeclareMathOperator*{\argmin}{arg\,min}

\newcommand{\tm}[1]{\textrm{#1}}
\newcommand{\rn}[1]{[\!\,{#1}\,\!]}
\definecolor{llgray}{RGB}{250,235,215} %antique white
\definecolor{lgray}{RGB}{220,220,220}

\begin{document}

%\markboth{Edgardo Solano-Carrillo}{Instructions for
%Typesetting Camera-Ready Manuscripts}

%\catchline{}{}{}{}{}
%
%%%%%%%%%%%%%%%%%%%%%%%%%%%%%%%%%%%%%%%%%%%%%%%%%%%%%%%%%%%%%%%%%%%%

\title{Can a single neuron learn predictive uncertainty?}

\author{Edgardo Solano-Carrillo}

\address{German Aerospace Center (DLR)\\ Institute for the Protection of Maritime Infrastructures\\
Bremerhaven, Germany\\
\texttt{Edgardo.SolanoCarrillo@dlr.de}}

\maketitle

%\begin{history}
%\received{(received date)}
%\revised{(revised date)}
%\accepted{(Day Month Year)}
%\comby{(xxxxxxxxxx)}
%\end{history}

\begin{abstract}
Uncertainty estimation methods using deep learning approaches strive against separating how uncertain the state of the world manifests to us via measurement (objective end) from the way this gets scrambled with the model specification and training procedure used to predict such state (subjective means) --- e.g., number of neurons, depth, connections, priors (if the model is bayesian), weight initialization, etc. This poses the question of the extent to which one can eliminate the degrees of freedom associated with these specifications and still being able to capture the objective end. Here,    
a novel non-parametric quantile estimation method for continuous random variables is introduced, based on the simplest neural network architecture with one degree of freedom: a single neuron. Its advantage is first shown in synthetic experiments comparing with the quantile estimation achieved from ranking the order statistics (specifically for small sample size) and with quantile regression. In real-world applications, the method can be used to  quantify predictive uncertainty under the split conformal prediction setting, whereby prediction intervals are estimated from the residuals of a pre-trained model on a held-out validation set and then used to quantify the uncertainty in future predictions --- the single neuron used here as a structureless ``thermometer'' that measures how uncertain the pre-trained model is. Benchmarking regression and classification experiments demonstrate that the method is competitive in quality and coverage with state-of-the-art solutions, with the added benefit of being more computationally efficient.    
\end{abstract}


\section{Introduction}

Estimating how uncertain artificial intelligence systems are of their predictions is crucial for their safe applications \cite{safe2016,ai_safety,Shi1,Shi2}. Quantifying uncertainty is then as important as designing good predictive models. It is an open problem, in part due to its ambiguous character: if predictions are interpreted as subjective opinions, evidence-based theory \cite{Sensoy, josang, Shi} typically measures uncertainty in entropic terms. On the other hand, if uncertainty is understood as a synonym for the variability of the predictive distribution, prediction intervals \cite{reviewPI} best summarize this in quantile terms.\newpage 

Prediction intervals express uncertainty in terms of \emph{confidence} probabilities, of which humans have a natural cognitive intuition \cite{cosmides, Juanchich2020} of guidance for decision making. As such, their use to quantify uncertainty is standardized across a wide variety of safety-critical regression applications, including medicine \cite{medPI}, economics \cite{econPI}, finance \cite{finPI}; as well as in the forecasting of electrical load \cite{energyPI}, solar energy \cite{solarPI}, gas flow \cite{gasPI}, wind power \cite{windPI}, and many other forecasting problems \cite{m4PI}.

In classification applications, there is less consensus on the use of a confidence-based (and hence intuitive) measure of uncertainty \cite{survey_classif2021}. For image classification, for instance, dozens of different uncertainty measures exist \cite{SQR}. Nevertheless, despite the diversity of methods to quantify uncertainty accross prediction categories, making uncertainty inferences has converged to a mainstream strategy: the same machine learning model that predicts a given target \emph{simultaneously} learns the associated uncertainties. These models are often underspecified \cite{underspec}, giving unreliable predictions under stress tests, and also under distribution shift \cite{Ovadia, Hendrycks}. This unreliability is therefore translated (by design) to how these models assess uncertainty. 

A different strategy is then considered in this work: \emph{model} how to predict, as usual, but \emph{measure} the associated predictive uncertainty during validation, using a confidence-based learning method. This model-agnostic approach to uncertainty estimation is aligned with rising trends in the post-hoc explainability of deep learning models \cite{rev_expl_2020}. That is, a predictive deep learning model is considered as a black box and a second system estimates how uncertain the black box is. Our main motivation in this work is finding an uncertainty estimator which is not part of the black box and therefore not having itself any associated uncertainty due to model specification. This leads us to the extreme case of a non-parametric quantile estimator consisting of a single neuron. A number of synthetic and real-world experiments demonstrate that the proposed quantile estimator has similar accuracy but better efficiency than some state-of-the-art methods.

The main contribution of this work is then:

\begin{itemize}
 \item A method for quantile estimation, which measures (using gradient descent) the predictive uncertainty of a pre-trained model in a held-out validation set, eliminating the bias associated with model specification.
 \item An attempt to cover, under the same umbrella, the predictive uncertainty in both regression and classification problems using a confidence-based method.
\end{itemize}

\section{Quantile estimation}
Let $E$ be a real random variable with distribution $F$, so that $\Pr(E\le \varepsilon)=F(\varepsilon)$. For any $p\in(0,1)$, a $p$-th quantile of $F$ is a number $r_p$ satisfying $F(r_p-)\le p \le F(r_p)$, where the left limit is $F(r_p-):=\lim_{z \uparrow r_p}F(z)$. For all \emph{continuous} distribution functions $F$, of interest here, this becomes
\begin{equation}\label{eq:Fp}
 F(r_p)=p.
\end{equation}
If $F$ is a strictly increasing function, then there is only one number satisfying \ref{eq:Fp}. It defines the quantile function $r_p=F^{-1}(p)$. It is our purpose in this work to introduce an efficient non-parametric method to estimate it, and apply it to typical regression problems with no discrete component of $F$.

\subsection{Proposed estimator.}\label{sec:PIM}
To be able to use a neural network, and at the same time obtain a non-parametric quantile estimator, a single neuron is considered whose weight $w_p$ coincides with the quantile to be learned. Using an independently drawn sample $\bm{\varepsilon} =(\varepsilon_1, \varepsilon_2,\cdots,\varepsilon_m)$ of $E$ of size $m$, this neuron activates itself to output the empirical distribution function $F_m(w_p)=\frac{1}{m}\sum_{i=1}^{m}\mathbbm{1}(\varepsilon_i \le w_p)$, making an error $\mathcal{L}(w_p)=[F_m(w_p)-p]^2$ in reaching its target $p$. After properly initializing $w_p$, this neuron is trained with gradient descent after smoothing the indicator function $\mathbbm{1}(\varepsilon_i \le w_p)\sim\sigma(\beta(w_p-|\varepsilon_i|))$ using a sigmoid $\sigma(x)=(1+\exp(-x))^{-1}$; this approximation becomes exact as $\beta\rightarrow\infty$.\footnote{Since $\nabla\mathcal{L}(w_p)\sim \beta$, in practice, the value of $\beta$ can be jointly selected with the learning rate $lr$. For most of the experiments in this work, $\beta=10^3$ with $lr=0.005$ work well.}

From Borel's law of large numbers, $F_m(w_p)$ almost surely tends to $F(w_p)$ for infinite sample size. In this limit, our neuron is trained by minimizing $\mathcal{L}(w_p)=[F(w_p)-p]^2$, so its weight $w_p$ converges to the global minimum $r_p$ by \ref{eq:Fp}. Therefore, the proposed quantile estimator is asymptotically consistent. For reasons that become clearer later, it is called a \emph{Prediction Interval Metric} (PIM). Further theoretical details and link to the source code reproducing the experiments may be found in the appendices.

\subsection{Comparison to estimation from the order statistics.}
Quantile estimation from ranking the order statistics is a pretty standard technique with at most $O(m)$ complexity. It considers the sample $\bm{\varepsilon} =(\varepsilon_1, \varepsilon_2,\cdots,\varepsilon_m)$ and starts by constructing the order statistics $\varepsilon_{(k)}$ as the $k$-th smallest value in $\bm{\varepsilon}$, with $k=1,\cdots,m$. If $mp$ is not an integer, then there is only one value of $k$ for which $(k-1)/m < p < k/m$; this is called the rank. Since $F_m(\varepsilon)=k/m$ for $\varepsilon_{(k)} \le \varepsilon < \varepsilon_{(k+1)}$, then a unique $p$-th quantile is estimated as $\varepsilon_{(k)}$. However, if $mp$ is an integer, an interval of $p$-th quantiles of $F_m$ exists with endpoints $\varepsilon_{(k)}$ and $\varepsilon_{(k+1)}$, the rank becoming a real-valued index. How to select a representative value from such an interval?

One posibility would be to take the midpoint $(\varepsilon_{(k)}+\varepsilon_{(k+1)})/2$. This is equivalent to Laplace's ``Principle of Insufficient Reason'' as an attempt to supply a criterion of choice \cite{jaynes_1957}, that is, since there is no reason to think otherwise, the events: the best representative value is $\varepsilon_{(k)}$ or the best representative value is $\varepsilon_{(k+1)}$, are equally likely. Hyndman \& Fan\cite{HandF} compiled a taxonomy of nine interpolation schemes used by a number of statistical packages. They all add to the arbitrariness of selection of a representative value. Since $\Delta F_m(\varepsilon) \sim 1/m$ as $\Delta \varepsilon\sim \varepsilon_{(k+1)}-\varepsilon_{(k)}$, this arbitrariness has a major impact for small sample size, as shown in Fig. \ref{rmse}, where the confidence interval function $I(p)= r_{(1+p)/2}-r_{(1-p)/2}$ of a standard normal random variable is  estimated using all interpolation methods provided by the {\small \textsc{numpy}} library. As observed, PIM does not have such a selection bias and can be more accurate for small sample size (as can also be demonstrated in experiments extending PIM to the classification domain, see appendix \ref{sec:PIM_classif}). 

\begin{figure}[t]
 \centering
 \includegraphics[scale=0.45]{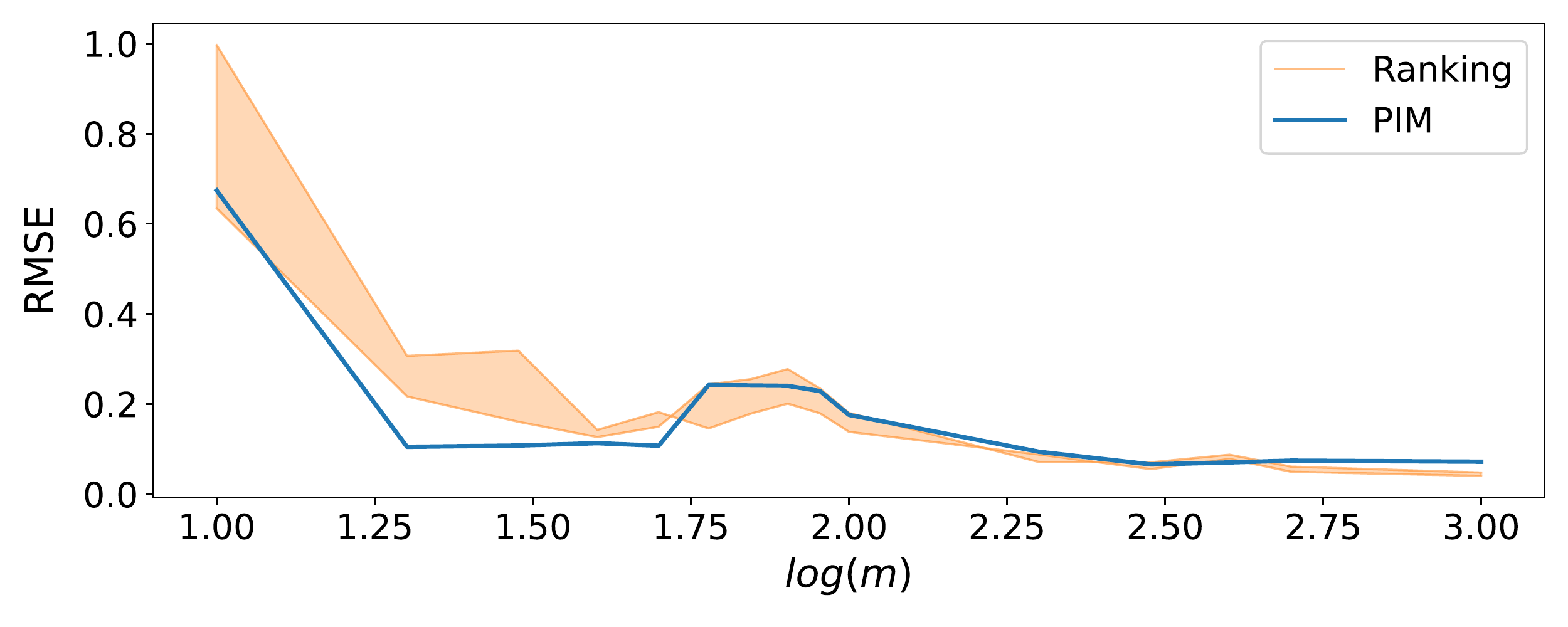}
 \caption{Root Mean Square Error (RMSE) between the estimated and exact confidence interval function $I(p) = r_{(1+p)/2}-r_{(1-p)/2}$ around the median of a standard normal random variable as a function of the log of sample size. The values of $p$ are in the range $[0.05,0.90]$ in steps of $0.05$.}
 \label{rmse}
\end{figure}

\subsection{Conditional quantiles}\label{sec:condq}
In regression analysis, one is interested in explaining the variations of a target random variable $Y$ taking values $y\in\mathbb{R}$ in terms of feature random variables $X$ taking values $x\in\mathbb{R}^d$. It is usually assumed, either implicitly or explicitly, that a deterministic map $f$ exists, explaining such variations as $y = f(x)+\varepsilon_{\tm{obs}}(x)$, up to some additive noise $\varepsilon_{\tm{obs}}(x)$ inherent to the data observation process. Empirically, this map is estimated by choosing a statistical model $\hat{f}(x)$ (e.g. a neural network), which approximates the target as
\vspace{-0.15cm}
\begin{equation}\label{y_mod}
 y = \hat{f}(x)+ \varepsilon(x).
\end{equation}
In so doing, the predictive model makes the error $\varepsilon(x) =\varepsilon_{\tm{obs}}(x)+\varepsilon_{\tm{epis}}(x)$ consisting of the aleatoric part $\varepsilon_{\tm{obs}}(x)$ and an epistemic  part $\varepsilon_{\tm{epis}}(x)=f(x)-\hat{f}(x)$ which entails an uncertainty due to the lack of knowledge of $f(x)$. This could be because we are not sure how to select $\hat{f}$ (model specification) or because the shape of $f$ for unexplored regions of feature space might be significantly different from that inferred from the training set (distributional changes).

The random variable $E$ defined previously is now conditioned on $X$, which is denoted as $E|X$. It will be understood to take the error values $\varepsilon(x)$ in \ref{y_mod}, and is relocated to satisfy $\tm{median}(E|X)=0$. We say that the errors are \emph{homoskedastic} if $E$ is independent of $X$, otherwise they are \emph{heteroskedastic}. There are then two ways to calculate the aleatoric uncertainty of the target variable $Y$: 
\begin{enumerate}
 \item Estimating the conditional quantile function $\mu_p(x)$ which assigns pointwise the smallest $\mu$ for which $\Pr(Y\le \mu|x)=p$ and, from this, computing the prediction intervals $[\hat{\mu}_{(1-p)/2}(x),\,\hat{\mu}_{(1+p)/2}(x)]$ quantifying the uncertainty of the target at confidence level $p$.
 \item  Estimating the conditional quantile function $r_p(x)$ of the error variable $E|X$ and computing the corresponding prediction intervals $[\hat{f}(x)-\hat{r}_{p}(x),\,\hat{f}(x)+\hat{r}_{p}(x)]$ quantifying the uncertainty of the target at confidence level $p$. This assumes that $\hat{f}$ is a good approximator of the median of $Y|X$.
\end{enumerate}

Approaches of type 1 are known as quantile regression. They enlarge the model $\hat{f}\rightarrow (\hat{f}_L, \hat{f}_U)$ to fit the endpoints of the prediction intervals, i.e. $\hat{f}_{L;p}(x)=\hat{\mu}_{(1-p)/2}(x)$ and $\hat{f}_{U;p}(x)=\hat{\mu}_{(1+p)/2}(x)$. Given a training set $\{(x_i,y_i):i\in \mathcal{I}\}$, we consider two state-of-the-art methods of this kind 

\begin{itemize}
 \item  Simultaneous Quantile Regression (SQR): this minimizes the average pinball loss $\frac{1}{|\mathcal{I}|}\sum_{i\in \mathcal{I}} l_p(\varepsilon_i)$  for $\varepsilon_i=y_i-\hat{f}_{L;p}(x_i)$ and for $\varepsilon_i=y_i-\hat{f}_{U;p}(x_i)$ simultaneously in the same model, where $l_p(\varepsilon_i)=p\,\varepsilon_i\,\mathbbm{1}(\varepsilon_i\ge0)+(p-1)\,\varepsilon_i\,\mathbbm{1}(\varepsilon_i <0)$. This is enough for our purpose, since it has less execution steps than the state-of-the-art SQR\cite{SQR}. Yet, it works better than standard quantile regression which estimates each quantile separately.
 \item Quality Driven (QD) method \cite{pearce}: In the training subset indexed by $\mathcal{C}=\{i: \hat{f}_{L;p}(x_i)\le y_i \le \hat{f}_{U;p}(x_i)\}$, this minimizes the captured mean prediction interval width (MPIW), which is expressed as $\tm{MPIW}_{\tm{capt}}=\frac{1}{|\mathcal{C}|}\sum_{i\in \mathcal{C}}[\hat{f}_{U;p}(x_i)-\hat{f}_{L;p}(x_i)]$, subject to the prediction interval coverage proportion (PICP) satisfying $\tm{PICP}:= |\mathcal{C}|/|\mathcal{I}| \ge p$. The rationale is that prediction intervals of good quality \cite{LUBE} have $\tm{MPIW}_{\tm{capt}}$ as small as posible and enough coverage. 
\end{itemize}

For approaches of type 2, the training set has to be split into two disjoint subsets: a proper training set $\{(x_i,y_i):i\in \mathcal{I}_1\}$ and a calibration (or validation) set $\{(x_i,y_i):i\in \mathcal{I}_2\}$. The proper training set is used to fit $\hat{f}(x)$, which is then evaluated on the validation set to compute the errors $\varepsilon_i:=\varepsilon(x_i)=y_i-\hat{f}(x_i)$ and their quantiles. This is the setting used in the split conformal  prediction literature \cite{Lei_2018}, where sample quantiles are estimated by ranking the order statistics. However, in this literature, a single $\hat{r}_p$ is obtained from $\{\varepsilon_i:i\in \mathcal{I}_2\}$. PIM may also be applied within this setting, obtaining $\hat{r}_p(x)$ which could vary with $x$.

In order to obtain variable $\hat{r}_p(x)$ with PIM, a different neuron $u_i$ has to be used for each position $\{x_i:i\in\mathcal{I}_2\}$. If the data-generating distribution is known (or may be properly approximated), this is used to sample extra targets $\{y_{i;k}: k\in\mathcal{J}\}$ not known to $\hat{f}$, so PIM estimates the quantiles from the neuron $u_i$ having access to the errors $\varepsilon_{i;k}=y_{i;k}-\hat{f}(x_i)$ for $k\in\mathcal{J}$. A synthetic example of this is shown in Fig. \ref{fig:toy}. If the data-generating distribution is unknown, but serial correlations are important (i.e. the order of $i$ in $\mathcal{I}$ matters), then PIM may be applied if relative rather than absolute positions in feature space are relevant. This is done by having $|\mathcal{T}|$ different neurons $u_j$ learn quantiles from samples $W_i=[\varepsilon_{i}, \varepsilon_{i+1}, \cdots, \varepsilon_{i+|\mathcal{T}|-1}]$ of the joint error distribution for $i=1,2,\cdots,|\mathcal{I}_2|-|\mathcal{T}|+1$, provided there are enough samples in the validation set. By rolling the window $W_i$, a new sample from the joint distribution is obtained, and the neuron $u_j$ is trained with all the errors observed at the $j$-th position of all windows. A real-world example of this is shown in Fig. \ref{fig:lstm}. These two examples are described in more detail next.

\begin{figure}[!htb]
    \centering
    \begin{minipage}{.48\textwidth}
        \centering
        \includegraphics[width=\linewidth]{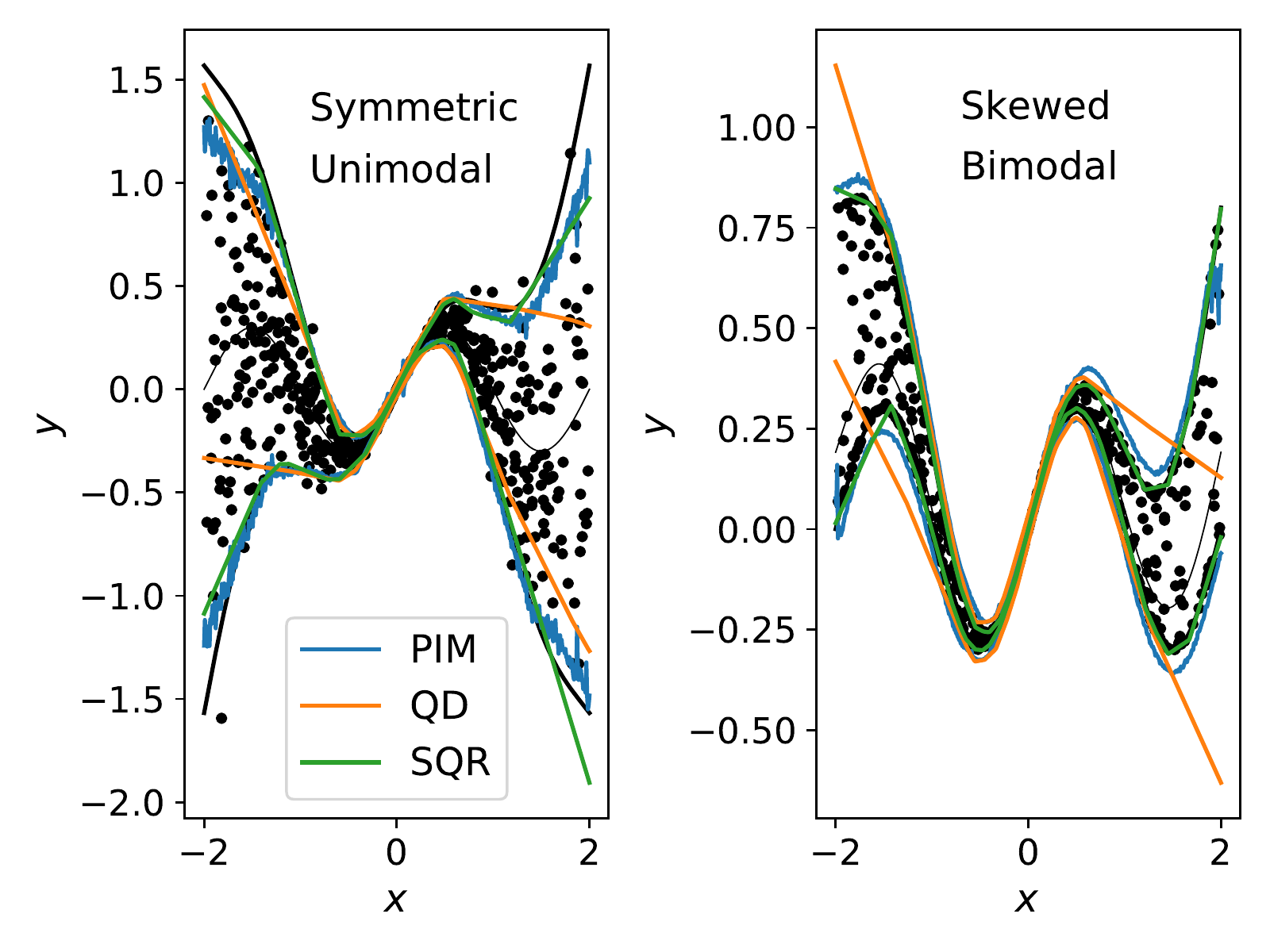}
    \end{minipage}\qquad
    \begin{minipage}{.4\textwidth}
        \centering
        \includegraphics[scale=0.34]{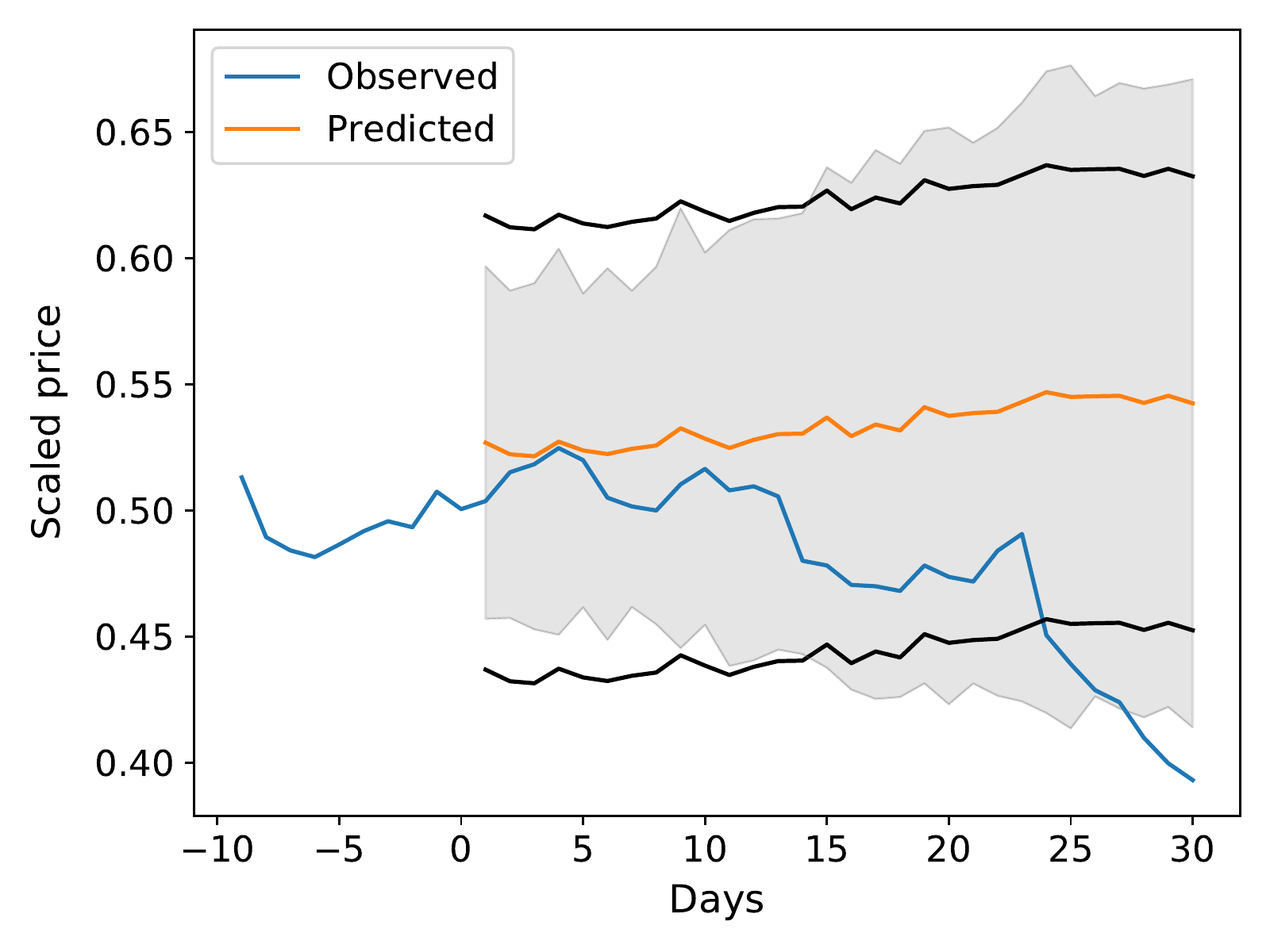}
    \end{minipage}

    \begin{minipage}[t]{.45\textwidth}
        \caption{Prediction intervals with $95\%$ confidence estimated by PIM, QD and SQR for synthetic data (black points) normally distributed and with a skewed Beta distribution. The black thick lines are the ideal boundaries, whereas the thin lines are the medians of the distributions.}
        \label{fig:toy}
    \end{minipage}\qquad
     \begin{minipage}[t]{.45\textwidth}
        \caption{Predicted and observed scaled price of the stock of General Electric for a horizon of $h=30$ days. The gray area are $95\%$ prediction intervals estimated by PIM; the black lines correspond to bounds based on a baseline assuming normally distributed errors.}
        \label{fig:lstm}
    \end{minipage}

\end{figure}

\textbf{Synthetic experiment}. The aim here is to compare the accuracy and computational efficiency of PIM against QD and SQR. For this, consider a one-dimensional data-generating process described by $y(x)=0.3\sin(x)+ \varepsilon_{\tm{obs}}(x)$ where $X\sim U(-2,2)$. For the error associated with observation, two cases are considered --- both having scale $\sigma(x)=0.2\,x^2$. The first is Gaussian error $E_{\tm{obs}}|X\sim N(0,\sigma^2(x))$, exemplifying a symmetric, unimodal distribution. The second is $E_{\tm{obs}}|X\sim \tm{Beta}(a,b,\tm{loc}=0, \tm{scale}=\sigma(x))$, exemplifying a skewed, bimodal distribution when $a<1$ and $b<1$. For concreteness, take $a=0.2$ and $b=0.3$. 

The conditional quantiles of the target may be expressed as $\mu_p(x)=0.3\sin(x)+\sigma(x)\mu_p$, where $\mu_p$ is the quantile function of the standardized error variable (computed by most statistical libraries for known distributions). From this, prediction intervals [$\mu_{(1-p)/2}(x)$, $\mu_{(1+p)/2}(x)$] may be calculated for the two cases of interest; these  are bounded by the black thick lines in Fig. \ref{fig:toy} for $p=0.95$. For visual aid of the symmetry/skeweness of the distributions, the median $\mu_{0.5}(x)$ is also shown as thin black lines. 

In a single trial of the experiment, a neural network $\hat{f}$ with 100 hidden units and output layer with \emph{one} unit is trained by sampling 500 pairs $P=\{(x_i,y_i): i\in\mathcal{I}\}$, shown as black points in Fig. \ref{fig:toy}. Since this experiment is synthetic, there is no need to split the training set. Instead, $\hat{f}$ is evaluated on a grid $G=\{x_j: j\in\mathcal{T}\}$ disjoint to $P$, partitioning $[-2,2]$ in 500 intervals of equal length. PIM is trained\footnote{For skewed distributions, two neurons independently learn $\hat{r}_{p}^L$ and $\hat{r}_{p}^U$ from $\varepsilon\le0$ and $\varepsilon>0$ respectively; the prediction intervals estimated as $[\hat{f}-\hat{r}_{p}^L, \hat{f}+\hat{r}_{p}^U]$.} on $G$ by sampling $|\mathcal{J}|=1000$ values of $y(x_j)$ for each $x_j$ in $G$, each neuron $u_j$ learning quantiles from $\{\varepsilon_{j,k}=y_k(x_j)-\hat{f}(x_j): k\in\mathcal{J}\}$. QD and SQR both train a neural network with 100 hidden units and output layer with \emph{two} units $(\hat{f}_L, \hat{f}_U)$. The model has the same hyperparameters as $\hat{f}$. However, for a fair comparison, the training set of $(\hat{f}_L, \hat{f}_U)$ is $P$ augmented with $|\mathcal{J}|$ more pairs disjoint to $G$. The results for a trial are shown in Fig. \ref{fig:toy}.

\begin{table*}[t]
\caption{Evaluating accuracy and efficiency of quantile estimation by three different methods. The results show median $\pm$ absolute median deviation over the trials. The notation $0.036\pm 0.001$ is simplified to $0.036\,(1)$.}
\label{tab:toy}
\vskip 0.1in
\begin{center}
\begin{small}
\begin{sc}
\begin{tabular}{cccc}
\toprule
Method & \# Parameters & Time & RMSE \\
\midrule
SQR & 402 & $0.080\,(9)$ & $\mathbf{0.09\,(3)}$\\
QD & 402 & $0.042\,(8)$ & $0.55\,(1)$\\
$\hat{f}+$ PIM & 301 & $\mathbf{0.036\,(1)}$ & $0.218\,(3)$\\
\bottomrule
\end{tabular}
\end{sc}
\end{small}
\end{center}
\vskip -0.1in
\end{table*}

The experiment is repeated for 10 trials. For each of them, the time taken to train + evaluate the models $(\hat{f}_L, \hat{f}_U)$ --- as well as training and evaluating $\hat{f}$ + train PIM --- is measured and normalized by the total duration of the 10 experiments. This together with the RMSE between estimations and ideal values is shown in Table \ref{tab:toy} for the case of normally distributed noise. As observed, the quantile estimation using PIM has less time and parameter complexity and thus more computationally efficient. In terms of accuracy, $\hat{f}+\tm{PIM}$ ranks in between SQR and QD despite the fact that $\hat{f}$ is trained with less data and has less parameters than $(\hat{f}_L, \hat{f}_U)$.

\textbf{Real-world experiment} The aim here is to demonstrate that PIM may be used in realistic contexts where varying prediction intervals are needed, specially when the uncertainty of interest is related to relative rather than absolute positions in feature space.  As an illustration, the uncertainty in the prediction of the stock price of General Electric is considered. A LSTM model $\hat{f}$ learns to map features of the last $T=10$ observations to the next $h=|\mathcal{T}|=30$ target close prices. This is done in a training set with the first 9840 samples of daily data from 1962 to 2001. The trained model $\hat{f}$ is evaluated on a validation set consisting of the next 4218 samples, where PIM learns prediction intervals corresponding to $h$ consecutive predictions, using neurons $u_j$ for $j\in \mathcal{T}$. These are placed around the predictions on the held-out \emph{test} set shown in  Fig. \ref{fig:lstm}. These are compared with a popular baseline \cite{hyndman}, consisting of bounds $\pm\, z_p \,\hat{\sigma}_h$ derived by assuming that the errors are normally distributed, with $z_p$ being the z-score. If the forecasts of all $h$ future prices in the test set are assumed to coincide with the average of the past $T$ observations (which is roughly the case in Fig. \ref{fig:lstm}), then it can be shown that $\hat{\sigma}_h=\hat{\sigma}\sqrt{1+1/T}$, where $\hat{\sigma}$ is the standard deviation of the $h$ error samples in the test set.

Apart from PIM having better coverage than the baseline and having narrower prediction interval widths at the beginning of the test sequence (hence better quality), this example shows how PIM captures the epistemic uncertainty resulting from $\hat{f}$ knowing better that predictions for tomorrow should be close (by continuity of $f$) to observations today, giving rise to the cone-shaped uncertainty region. This information is cheap: while the inference time of the LSTM is about $3.9$ sec, PIM only takes about $0.4$ sec to obtain the prediction intervals from the validation set. 

\section{Using a single neuron to estimate real-world uncertainty}\label{sec:reg}
It has been shown that training a model $\hat{f}$ that learns the mean --- hopefully equal or close to the median --- of the target distribution and using PIM to estimate its quantiles is more efficient and has similar accuracy than having a bigger model $(\hat{f}_L, \hat{f}_U)$ learning the boundaries of the prediction intervals directly. Also, the quantile estimation for small sample size can be more accurate using PIM than ranking the order statistics. With all these benefits, would you use it in your applications?

The answer to this depends on the dataset. Since the training set has to be split into a proper training set $\{(x_i,y_i): i\in\mathcal{I}_1\}$ to fit $\hat{f}$ and a validation set $\{(x_i,y_i): i\in\mathcal{I}_2\}$ to train PIM, the resulting size $|\mathcal{I}_2|$ of the validation set might not be enough for PIM to get accurate results. Also, the amount of heteroskedasticity in the dataset may invalidate using a single neuron learning $\hat{r}_p$. The effect of these two factors is investigated next for real-world datasets, having the results from QD as a baseline. That is, we follow the experimental protocol established by Lobato \emph{et al}\cite{lobato} for the popular UCI regression benchmark. This assigns $90\%$ of the data (from 10 different datasets) for training uncertainty estimation models and $10\%$ for testing them, in an ensemble of mostly 20 random shuffles of the train-test partition. 

To apply PIM, the $80\%$ of the resampled training set of each dataset is used to train the nominal neural network $\hat{f}$, which is evaluated on the remaining $20\%$, where PIM is trained from the corresponding prediction errors. The best between [$\hat{f}-\hat{r}_p, \hat{f}+\hat{r}_p$] and [$\hat{f}-\hat{r}_p^L, \hat{f}+\hat{r}_p^U$], in addressing coverage and quality in the \emph{test} sets, is chosen. Therefore, the test MPIW of the QD method is compared to either $2\hat{r}_p$ or $\hat{r}_p^L+\hat{r}_p^U$, depending on which is smaller, and which PICP (which is nothing but the $F_m$ of section \ref{sec:PIM}) is closer to the nominal $p=0.95$. As in the synthetic experiments of the previous section, note that the QD model, besides having more outputs $\hat{f}_L$ and $\hat{f}_U$ (hence more weights), is trained on more data than $\hat{f}$. 

A measure of heteroskedasticity of the datasets is needed in order to better understand the resulting estimations. For this, a White test \cite{White} is done in every fold used to train PIM. This looks for linear dependency of the variance $\mathbb{E}(\xi^2)$ of residuals $\xi$ (from a linear regression of $y$ on $x$) on all features in $x$ and their interactions. The proportion of significant tests in the ensemble, according to the p-value of the F-statistic, is denoted by $P_{\tm{SIG}}$ and reported as a percentage. This gives the percentage of times that the null hypothesis of homoskedastic residuals is rejected; giving a sense of residual variability among the validation folds of the ensemble but not how ``strong'' that variability is within a fold. 

To quantify the degree of variability of the residuals $\xi_i$ used for the White tests, the \emph{normalized} power spectral entropy (PSE) of such residuals is proposed
\begin{equation}
 \tm{PSE}=-\frac{1}{\log |\mathcal{I}_2|}\sum_{i\in\mathcal{I}_2} p_i \log p_i,
\end{equation}
where $p_i=|\xi_i|^2/\sum_i |\xi_i|^2$ normalizes the square amplitude of the $i$-th spectral component of $\xi$ (found by a fast Fourier transform). The intuition is that \emph{patterns} in $\xi$ have a low entropy $\tm{PSE}\rightarrow0$ whereas homoskedastic-like residuals (e.g. white noise) have high entropy $\tm{PSE}\rightarrow1$.

\begin{table*}[t]
\caption{\emph{Test} quality metrics for prediction intervals in relevant datasets: mean $\pm$ std dev over the ensemble. Best results in bold, compared according to the criteria and results of the QD authors \cite{pearce}, that is: if $\tm{PICP}\ge0.95$ for QD and PIM, both were best for PICP, and best MPIW is given to the smallest MPIW. If $\tm{PICP}\ge0.95$ for neither or for only one, largest PICP was best, and MPIW assessed if the one with larger PICP also has smallest MPIW.} 
\label{tab:PIMvsQD}
\vskip 0.1in
\begin{center}
\begin{small}
\begin{sc}
\begin{tabular}{cccccccc}
\toprule
%Data set & Loss & PICP & MPIW \\
\multirow{2}{*}{Dataset} &
\multirow{2}{*}{$P_{\tm{SIG}}$} &
\multirow{2}{*}{PSE} &
\multicolumn{2}{c}{PICP} &
\multicolumn{2}{c}{MPIW} \\
&  &  & QD-Ens & PIM-Ens & QD-Ens & PIM-Ens\\
\midrule
Yacht &$65$& $0.80\,(6)$& $\mathbf{0.96\,(1)}$ & $0.90\,(6)$& $\mathbf{0.17\,(0)}$ &$0.26\,(7)$  \\
%\rowcolor{lightgray}
Boston&$70$& $0.87\,(6)$& $\mathbf{0.92\,(1)}$ & $0.84\,(8)$& $1.16\,(2)$ &$0.7\,(1)$  \\
Energy &\cellcolor{llgray}$100$& \cellcolor{llgray}$0.80\,(2)$& \cellcolor{llgray}$\mathbf{0.97\,(1)}$ &\cellcolor{llgray} $\mathbf{0.95\,(2)}$& \cellcolor{llgray}$0.47\,(1)$ &\cellcolor{llgray}$\mathbf{0.21\,(4)}$ \\
Concrete & $100$& $0.74\,(2)$& $\mathbf{0.94\,(1)}$ & $0.90\,(4)$& $1.09\,(1)$ &$1.04\,(9)$  \\
Red Wine &$75$& $0.78\,(2)$& $\mathbf{0.92\,(1)}$ & $0.82\,(9)$& $2.33\,(2)$ &$1.9\,(3)$  \\
%\rowcolor{lightgray}
Kin8nm &\cellcolor{llgray}$100$& \cellcolor{llgray}$0.72\,(1)$& \cellcolor{llgray}$\mathbf{0.96\,(0)}$ & \cellcolor{llgray}$\mathbf{0.95\,(1)}$& \cellcolor{llgray}$1.25\,(1)$ & \cellcolor{llgray}$\mathbf{1.17\,(6)}$  \\
Power Plant & \cellcolor{llgray}$90$& \cellcolor{llgray}$0.84\,(5)$& \cellcolor{llgray}$\mathbf{0.95\,(0)}$ & \cellcolor{llgray}$\mathbf{0.95\,(1)}$& \cellcolor{llgray}$\mathbf{0.86\,(0)}$ & \cellcolor{llgray}$\mathbf{0.87\,(2)}$ \\
%\rowcolor{lightgray}
Naval &\cellcolor{llgray}$100$& \cellcolor{llgray}$0.88\,(1)$& \cellcolor{llgray}$\mathbf{0.98\,(0)}$ & \cellcolor{llgray}$\mathbf{0.95\,(1)}$&\cellcolor{llgray} $0.28\,(1)$ &\cellcolor{llgray}$\mathbf{0.23\,(9)}$ \\
Protein &$100$& $0.66\,(0)$& $\mathbf{0.95\,(0)}$ & $0.94\,(0)$& $\mathbf{2.27\,(1)}$ &$2.65\,(1)$  \\
Song Year &$100$& $0.86\,(\cdot)$& $\mathbf{0.96}\,(\cdot)$ & $\mathbf{0.95}\,(\cdot)$& $\mathbf{2.48}\,(\cdot)$ &$3.12\,(\cdot)$  \\
\bottomrule
\end{tabular}
\end{sc}
\end{small}
\end{center}
%\vskip -0.1in
\end{table*}

The results of the comparison with QD are shown in Table \ref{tab:PIMvsQD}, where -ENS is appended to the acronyms of the methods to mean that the results are averages over the ensemble. The first observation is that, despite all datasets being heteroskedastic, the uncertainty estimations made by PIM in the validation sets generalize well into the test sets (better or similar to QD in the shaded cases).

The cases where PIM fails to converge to the desired PICP are those with significant variability among the validation folds (low $P_{\tm{SIG}}$) or appreciable presence of patterns in the errors (low PSE), as expected. Data size is also important since \emph{Kin8nm} is comparable to \emph{Concrete} in terms of $P_{\tm{SIG}}$ and PSE, but the former is more than 8 times bigger than the latter. This data-size dependence is evident from the table, where PIM has good performance mostly for the lower half datasets.\footnote{The datasets are ordered in size from top (308 samples) to bottom (515.345 samples).} For the biggest dataset, PIM does not excel presumably due to better calibrated predictions of QD's $(\hat{f}_L,\hat{f}_U)$ over the $\hat{f}$ feeding PIM --- expected from the flexibility of the former in terms of more network weights and more data to train them.

The second observation is that, in the successful cases, the test PICP achieved by PIM coincides with the intended confidence level $p=0.95$. As seen, this is not necessarily the case for QD, since it targets $\tm{PICP}\ge p$. However, for a \emph{continuous} target variable $Y$ --- as in many regression problems of practical interest --- $\tm{PICP}=p$ must be (asymptotically) satisfied at the $p$-quantile. Restricting $\tm{PICP}\ge p$ in these cases may lead to the same quantile estimation describing different confidence levels; the quantile-crossing phenomenon \cite{Takeuchi} that should be avoided.

The results above show that using PIM with a single neuron may give accurate estimation of high-quality prediction intervals even for heteroskedastic datasets with \emph{weak} serial correlations (i.e. patterns) of the prediction error samples. Therefore, in the exploratory data analysis of a given application, heteroskedasticity tests in the dataset may reveal whether or not to leverage from the efficiency of using a single neuron for uncertainty estimation.

\section{Related work}
The split conformal prediction literature \cite{Lei_2018} uses a setting similar to PIM, estimating quantiles by ranking the order statistics of prediction errors from a given $\hat{f}$. They give finite sample coverage guarantees by assuming \emph{exchangeability} of the training samples, which may not apply, for instance, for non-stationary stochastic processes typically found in real-world time series. Recently, Romano \emph{et al}\cite{Romano_2019} have extended this framework to produce varying $\hat{r}_p(x)$ by combining the split conformal prediction formalism with quantile regression from $(\hat{f}_L, \hat{f}_U)$. As shown, PIM has the advantage of being more flexible to high-quality quantile estimation from small sample sizes compared to ranking the order statistics. This may be important in those cases for which the size of the validation set is a small fraction of the training set.

Recent research on uncertainty estimation focuses on studying the different sources of uncertainty \cite{survey_uncertainty2021,review2021}. Prediction intervals capture the aleatoric uncertainty associated to the noisy data observation process \cite{pearce, SQR, NormSplit}. Epistemic uncertainty involves model specification and data distributional changes; the latter maintaining recent interest \cite{Malinin, hafner2019, SQR, Li_2020_CVPR, dist_aware, Charpentier, latent_unc_2020}; the former being less studied. Examples in deep learning of uncertainty due to model specification include network-depth uncertainty \cite{DepthUC} and uncertainty over the number of nodes in model selection \cite{mod_sel19}.

The incompleteness in problem formalization behind machine learning models (intimately connected to model specification) leads to a need for their interpretability \cite{towards_interp}. This need became more urgent in 2016, when the European Parliament published the General Data Protection Regulation, demanding (among other clauses) that, by May 2018, all algorithms have to provide ``meaningful explanations of the logic involved'' when used for decision making significantly affecting individuals (right of people to an explanation). Consequently, techniques to explain AI models started to permeate the literature. Ribeiro \emph{et al}\cite{Ribeiro16} made a case for model-agnostic interpretability of machine learning; while Rudin\cite{Rudin2019} argued in favor of designing predictive models that are themselves interpretable. Different reviews arose \cite{rev_interpret, surv_explain} to clarify concepts and classify the increasing body of related research. 

The current understanding \cite{rev_expl_2020} is that a model is interpretable if, by itself, is understandable (e.g. linear/logistic regression, decision trees, $K$-nearest neighbors, rule-based learners, Bayesian models). If not, it needs post-hoc explainability (e.g. tree ensembles; SVM; multi-layer, convolutional and recurrent neural networks). Post-hoc explainability is done by feature relevance analysis or visualization techniques, but most often by a second \emph{simplified}, and hence interpretable, model  which mimics its antecedent. Our approach to uncertainty estimation goes along lines similar to the latter: a predictive model $\hat{f}$ is considered as a black box and a second system (a single neuron) estimates how uncertain the black box is.

Although using a model to learn from a black box has been explored in a context related to uncertainty, e.g. calibration of neural networks using Platt scaling \cite{calib_reg}, it was not until recently that such a method is directly used to upgrade any black-box predictive API with an uncertainty score \cite{brando2020}. However, their wrapper is based on deep neural networks and hence is not interpretable. PIM \emph{is not} a parametric model, but uses a globally interpretable neural network with a single unit.

\section{Conclusion}
In this work, a non-parametric method to estimate predictive uncertainty of a pre-trained model is introduced. The method is competitive with state-of-the-art solutions in quality, with the additional benefit of giving uncertainty estimates more efficiently (i.e. no need to add extra layers or outputs to a predictive model). Although the method does not \emph{predict} the uncertainty of new data samples based on their feature values --- perhaps making it less attractive, as it deviates from the established machine learning paradigm --- it does give a sense of a safer uncertainty estimation, just because it does not inherit the learning biases of the predictive models. This makes it suitable for rather \emph{explaining} how uncertain a given model (treated as a black-box) is, and hence serving as a reliable guide to decision-making. Extensions of the method to estimate uncertainty in the classification setting is given in appendix \ref{sec:PIM_classif}.

\section*{References}
\noindent

\bibliography{pim}
\interlinepenalty=10000
\bibliographystyle{ws-ijufks}

\newpage
\beginsupplement 
\section*{Appendices}
A theoretical ground is given here to key aspects of the paper together with extra details about the numerical experiments\footnote{The source code can be found at \url{https://github.com/sola-ed/pim-uncertainty}}. Moreover, PIM is applied to binary classification in order to give more support to the claim that it can be more accurate than ranking the order statistics for small sample size. This is done, on real-world datasets, by comparing confidence intervals for the accuracy of classification, as estimated by PIM, with the corresponding estimations using the Bootstrap method.

The material in the following is organized as follows: section \ref{sec:PIM_theory} summarizes the main theoretical assumptions behind PIM, making it an asymptotically consistent estimator. A heuristic for identification of finite-sample convergence is discussed in section \ref{sec:PIMfinite}. Details about the regression experiments using the UCI datasets are given in section \ref{sec:UCI_reg}. Finally, section \ref{sec:PIM_classif} applies PIM in the classification context. It starts in section \ref{sec:bin_classif} with the problem formulation and proceeds with the experimental results comparing PIM with the Bootstrap method. A proof of the main theoretical result of the section is given in \ref{sec:classif_theory}, and details of the experiments are found in section \ref{sec:classif_exp}.

\subsection{PIM as a consistent estimator}\label{sec:PIM_theory}
The main result is summarized in Theorem \ref{thm:1} below. In order to prove it, we go in steps by first showing that for the distribution functions of interest, the $p$-th quantile is unique, given the confidence level $p$. This uniqueness guarantees that the loss function of PIM has asymptotically only one minimum and then gradient descent will converge to it, given small enough learning rates. The uniqueness is proved in the following lemma:

\begin{lemma}\label{unique}
 Let $F(\varepsilon)=\Pr(E\le \varepsilon)$ be a strictly increasing and continuous distribution function and $p\in(0,1)$. Then, the $p$-quantile $r_{p}$ is unique.
\end{lemma}
\begin{proof}
 A $p$-quantile of $F$ is a number $r_{p}$ satisfiying $F(r_{p})\le p$ and $F(r_{p}+\epsilon)\ge p$, for $\epsilon\rightarrow0^{+}$ \cite{bookST}. Since $F$ is continuous, Bolzano's theorem states that there is at least one point in the interval $[r_{p},r_{p}+\epsilon]$ where $F(r_{p})-p=0$. That there is only \emph{one} such point clearly follows from $F$ being strictly increasing. Therefore, as $\epsilon\rightarrow0^{+}$, $r_{p}$ becomes the unique value where $F(r_{p})=p$.
\end{proof}

\begin{lemma}\label{borel}
 Let $\lbrace\varepsilon_i\rbrace$, with $i=1,2,\cdots,m$, be a sequence of independent draws of the random variable $E$, according to the distribution $F(\varepsilon)=\Pr(E\le \varepsilon)$. With $\mathbbm{1}$ being the indicator function, define $F_m(r_p)=\tfrac{1}{m}\sum_{i=1}^{m}\mathbbm{1}(\varepsilon_i \le r_{p})$. Then, for all $r_{p}$ and with probability one, $F_m(r_p)$ converges to $F(r_p)$ in the limit $m\rightarrow\infty$.
\end{lemma}
\begin{proof}
 This is Borel's law of large numbers.
\end{proof}

\setcounter{theorem}{0}
\begin{theorem}\label{thm:1}
 Let $F(\varepsilon)=\Pr(E\le \varepsilon)=\int_{-\infty}^{\varepsilon}\rho(x)dx$ be a strictly increasing and continuous error distribution function associated to the random variable $E$, with $\rho$ being the corresponding probability density function. If $m$ samples are independently drawn from it, then PIM (with $\beta\rightarrow\infty$) evaluated on these samples, converges to the unique value $r_{p}$ for which $F(r_{p})=p$, when $m\rightarrow\infty$.
\end{theorem}
\begin{proof}
 In the limit $\beta\rightarrow\infty$, the $F_m(r_{p})$ in PIM coincides with the $F_m(r_{p})$ in Lemma \ref{borel}. Using this Lemma, the loss function in PIM is asymptotically $\mathcal{L}_{p}(\varepsilon)=(F(\varepsilon)-p)^{2}$. Its gradient is $\nabla_{\varepsilon}\mathcal{L}_{p}(\varepsilon)=2[F(\varepsilon)-p]\,\rho(\varepsilon)$. Since $\varepsilon$ is in the support of $E$, $\rho(\varepsilon)\neq0$, then gradient descent, with a small enough learning rate, leads PIM to converge to the value of $\varepsilon$ for which $F(\varepsilon)-p=0$. By Lemma \ref{unique}, there is only one such value, being the $p$-quantile $r_{p}$.
\end{proof}

\subsubsection{Convergence for finite validation sets}\label{sec:PIMfinite}
It is observed in the numerical experiments that the loss in PIM smoothly decreases and saturates about a small value. By using early stopping during optimization, the optimal value $\hat{r}_{p}$ is taken as the point where this saturation takes place. It is argued in this section why such heuristic approach makes sense. For this, it is convenient to think of the current value $w_p$ of the weight of the single neuron as following a trajectory parameterized by the epochs.

In practice, $w_{p}$ is updated when the optimizer processes a batch and, at the end of a training epoch, all batches have been processed. The training epochs can then be thought of as values achieved by a continuous variable $t$, which changes as $w_{p}$ goes from its initial value, along a \emph{smooth} trajectory $w_p(t)$, to the optimal value $\hat{r}_{p}$. Withouth loss of generality, it is supposed that these trajectories have no turning points, i.e. they monotonically increase or decrease the initial value $w_p(0)$ towards $\hat{r}_{p}$. Furthermore, the rate at which this happens is bounded:
\begin{equation}\label{lin}
 |\nabla_t w_p(t)| \le c_{p},\hspace{0.5cm}\tm{with}\;\;0<c_{p}<\infty.
\end{equation}

From the proof of Theorem \ref{thm:1} for infinite sample size, $\nabla_{\varepsilon}\mathcal{L}_{p}(\varepsilon)=2[F(\varepsilon)-p]\,\rho(\varepsilon)$, so from \eqref{lin},
\begin{equation}\label{gradL}
|\nabla_{t}\mathcal{L}_{p}(t)|\le 2c_{p}|F(w_p(t))-p|\,\rho(w_p(t)). 
\end{equation}
A hypothetical algorithm, running with infinite validation set, will start at $t=0$, from $w_p(0)$, with sucessive updates generated (assuming a plain SGD optimizer) as
\begin{equation}\label{trajI}
 w_p(t+dt)=w_p(t)-\eta\nabla_{t}\mathcal{L}_{p}(t),
\end{equation}
where $\eta$ is the learning rate and $dt = B/m$, with $B$ and $m$ being the batch and validation set sizes, respectively. The sizes $B$ and $m$ can be selected so that $dt$ is fixed, and arbitrarily small, when $B\rightarrow\infty$ and $m\rightarrow\infty$. 

The convergence of PIM to the optimal value $\hat{r}_{p}$ can be considered, in practical terms, as related to the saturation of the loss function $\mathcal{L}_{p}(t)$. Given a small enough tolerance $\sigma_{p}$, the algorithm is said to converge to $\hat{r}_{p}$ at epoch $t_*$ if $|\nabla_{z}\mathcal{L}_{p}(t_*)|\leq\sigma_{p}$, at which point the loss has saturated. If PIM is stopped at $t_*$, the error committed in estimating the $p$-quantile of $F$ is, up to first order in $\sigma_{p}$,
\begin{equation}
 |r_{p} -\hat{r}_{p}|=\dfrac{\sigma_{p}}{2c_{p}[\rho(r_{p})]^{2}},
\end{equation}
which is obtained by evaluating \eqref{gradL} at $t_*$, writing $\hat{r}_{p}=w_p(t_*)$, and expanding $F$ and $\rho$ around $r_{p}$, giving $|\nabla_{t}\mathcal{L}_{p}(t_*)|\leq\sigma_{p}$.

\paragraph{Finite validation sets.}In this case, the trajectories are not generated by \eqref{trajI} anymore. Here, $dt$ is not arbitrarily small, i.e. $\min \lbrace dt\rbrace = 1/m$, which happens when a batch contains only one data sample.\footnote{In practice, the batch size was taken to be equal to the sample size though in order to exploit the asymptotic properties behind PIM.} The trajectories are still considered smooth and with bounded speed, but now the values of $w_p$ updated by PIM are more sparse. These trajectories are generated by the loss $\mathcal{L}_{p}^{m}(t)=[F_m(t)-p]^2$, i.e.
\begin{equation}
 w_p(t+dt)=w_p(t)-\eta\nabla_{t}\mathcal{L}_{p}^{m}(t).
\end{equation}
Assuming the same constants $c_{p}$ serve as upper bounds to all the possible speeds,
\begin{equation}\label{gradLF}
 |\nabla_{t}\mathcal{L}_{p}^{m}(t)|\le 2c_{p}|F_m(w_p(t))-p|\,|F_m^{\prime}(w_p(t))|. 
\end{equation}
Saturation of the loss is understood as making \eqref{gradLF} as small as possible. Clearly, since the gradient $F_m^{\prime}(w_p(x))$ is bounded and does not vanish, this saturation happens at the epoch $t_*$ of closest approach between $F_m(w_p(t))$ and $p$, that is, 
\begin{equation}
 t_* = \argmin_{t\in[0,\infty)}|F_m(w_p(t))-p|.
\end{equation}
Again, denoting $\hat{r}_{p}=w_p(t_*)$, and using the triangle inequality,
\begin{equation}
|F_m(\hat{r}_{p})-p|\leq |F_m(\hat{r}_{p})-F(\hat{r}_{p})| + |F(\hat{r}_{p})-p|,
\end{equation}
the right-hand side approaching zero, by Lemma \ref{borel} and Theorem \ref{thm:1}, as more data is considered in the validation set. This explains why early stopping was used throughout the numerical experiments, by automatically detecting $t_*$ and retrieving the corresponding $\hat{r}_{p}$.

\subsection{Details of models on UCI regression datasets}\label{sec:UCI_reg}
The baseline model has one relu-activated hidden layer with 50 units, except for \emph{Protein} and \emph{Song Year}, having 100 units. The ensembles are 20 repetitions of the experiments, except for \emph{Protein} and \emph{Song Year}, for which only 5 and 1 repetitions are considered, respectively. Hyperparameter optimization is done using the Hyperband tuner in Keras. For this, two protocols were tried and the best of the two, for each dataset, reported: 
\begin{enumerate}
 \item Optimization of learning rate, decay rate, weight-initialization variance, and dropout rate (using the Adam optimizer).
 \item Optimization of weight-initialization variance, weight decay, initial learning rate, and the decay rate of its subsequent exponential decay in a learning rate schedule (using the AdamW optimizer).
\end{enumerate}
In both protocols, the mean square error loss is used. However, in the second protocol, the CWC value \cite{reviewPI} is added to the metric used in the validation set for model selection in the hyperparameter optimization process. This value is calculated as $\tm{CWC}= \tm{NMPIW}(1+\gamma e^{-\eta(\tm{PICP-p})})$, where NMPIW is the MPIW normalized to the range of the target variable, $\gamma=\mathbbm{1}(\tm{PICP}<p)$ and $\eta$ is a constant taken as 0.1. The PICP and MPIW are calculated by PIM.

\subsection{PIM for classification}\label{sec:PIM_classif}

In classification problems, the target $Y$ is a \emph{discrete} random variable, but these can be framed so that the prediction error $E|X$ is still a continuous random variable accessible to PIM. The aim of this section is twofold:

\begin{itemize}
 \item Give additional demonstration that PIM can be more accurate for small sample sizes than ranking the order statistics, by using real-world datasets for binary classification.
 \item Demonstrate that using PIM to estimate confidence intervals for the accuracy of a classifier is more efficient than standard computations based on the Bootstrap method. 
\end{itemize}
To better illustrate the problem consider a sample $\{\hat{f}(x_i)\in[0,1]: i\in\mathcal{I}_2\}$ of predictions from a binary classifier. Denoting by $[\![y]\!]$ the rounding operation, the accuracy of the classifier is
\begin{equation}\label{acc}
 \tm{ACC} = \dfrac{1}{|\mathcal{I}_2|}\sum_{i\in\mathcal{I}_2}\mathbbm{1}([\![\hat{f}(x_i)]\!]=y_i).
\end{equation}
How do we estimate a confidence interval for the accuracy? The simplest way is by using the normal approximation to the binomial result: 
\begin{equation}\label{dacc_n}
 \delta_p (\tm{ACC})_{\mathcal{N}}=z_p \sqrt{\hat{\mu}_{\tm{ACC}}(1-\hat{\mu}_{\tm{ACC}})\,/\,|\mathcal{I}_2|},
\end{equation}
where $z_p$ is the z-score and $\hat{\mu}_{\tm{ACC}}$ is an estimation of the mean $\mu_{\tm{ACC}}$ of the distribution of accuracies. Clearly, using ACC in \eqref{acc} as a substitute for $\hat{\mu}_{\tm{ACC}}$ is rough; that is the standard way of getting confidence intervals for accuracy from one sample of predictions.

The standard estimation can be improved by resampling the proper training and validation sets, $\mathcal{I}_1$ and $\mathcal{I}_2$, respectively. That is, by the de Moivre-Laplace central limit theorem, all the so-obtained values of ACC are asymptotically normally distributed around $\mu_{\tm{ACC}}$, so their mean $\hat{\mu}_{\tm{ACC}}$ is an unbiased estimator of $\mu_{\tm{ACC}}$. When the resampling is done with repetition (a.k.a. the Bootstrap method), in order to allow for enough data, confidence intervals can be estimated by ranking the order statistics of all ACC, instead of using \eqref{dacc_n}.

One of the main observations in this section is that, provided that $\hat{f}$ is well calibrated, PIM can estimate confidence intervals $\delta_p (\tm{ACC})$ for the accuracy of a classifier, which are good estimates even when using a single sample of predictions. This is formalized and proved in section \ref{sec:classif_theory}. Since it would not need to resample the validation set, this makes PIM more efficient than the Bootstrap method. Experiments comparing PIM with the two methods above are described next.

\subsubsection{Benchmaring experiments}\label{sec:bin_classif}
For binary classification, the target $Y$ is either 0 (negative class) or 1 (positive class). Predictive models capture this by making predictions $\hat{f}(x)\in[0,1]$. According to the chosen threshold $\tau$ (here $\tau=1/2$), these are positive predictions if $\hat{f}(x)>\tau$, otherwise they are negative. Furthermore, by comparing with the corresponding ground truth, each prediction may be categorized as true negative, true positive, false negative, or false positive; denoted, respectively, by the index $l\in\lbrace\tm{TN, TP, FN, FP}\rbrace$.

Relevant metrics of model performance are derived from the classification rates $R_k$. For instance, the accuracy can be written as $\tm{ACC}=p_{\tm{N}}R_{\tm{TN}}+p_{\tm{P}}R_{\tm{TP}}$, where $p_{\tm{N}}$ ($p_{\tm{P}}$), are the negative (positive) class proportions in the validation set. Confidence intervals for accuracy are then obtained as
\begin{equation}\label{dacc}
 \delta_p (\tm{ACC})=p_{\tm{N}}\,\delta_p(R_{\tm{TN}})+p_{\tm{P}}\,\delta_p(R_{\tm{TP}}),
\end{equation}
in terms of confidence intervals for the classification rates $\delta_p(R_l)$. The latter are estimated by PIM after estimating quantiles from the error samples $\varepsilon_l(x)=y_l-\hat{f}(x)$ observed in the validation set, where $y_l$ is the ground truth label if $l$ refers to a \emph{true} prediction, otherwise $y_l=\tau$. For this, four neurons $s_l$ are trained in parallel until the optimal weights $\hat{r}_p^l$ estimate the desired quantiles.

\setlength{\tabcolsep}{3.2pt}
\begin{table*}[t]
\caption{Comparison of Bootstrap vs PIM estimation of $95\%$ confidence interval widths for the accuracy of classification algorithms on UCI data sets: median $\pm$ median absolute deviation over the ensemble. Binomial estimates according to \eqref{dacc_n} are in the last column $2\delta_p$(ACC)$_{\mathcal{N}}$.}
\label{PIMvsBS}
\begin{center}
\begin{small}
\begin{sc}
%\begin{tabular}{c@{\hskip -0.1mm}cccccccc}
\begin{tabular}{cccccccc}
\toprule
\multirow{2}{*}{Dataset} &
\multirow{2}{*}{Size} &
\multirow{2}{*}{$p_{\tm{N}}$} &
\multirow{2}{*}{$\hat{\mu}_{\tm{ACC}}$} &
\multirow{2}{*}{$\Delta_{\tm{calib}}$} &
\multicolumn{2}{c}{2$\delta_p$(ACC)} &
\multirow{2}{*}{2$\delta_p$(ACC)$_{\mathcal{N}}$} \\
%\multirow{2}{*}{Better?}\\
& & & & & BS-Ens & PIM-Ens & \\
\midrule
Sonar & $208$& 0.47& 0.87& $0.00$&$0.31 \pm0.07$ & \cellcolor{llgray}$0.16 \pm0.07$&  \cellcolor{llgray}0.23  \\
Heart Disease & $303$& 0.46& 0.86& $0.00$& \cellcolor{llgray}$0.35 \pm0.05$ & \cellcolor{llgray}$0.25 \pm0.09$&  \cellcolor{llgray}0.17\\
Ionosphere & $351$& 0.36&0.88& $0.00$& \cellcolor{llgray}$0.19\pm 0.04$ & \cellcolor{llgray}$0.19 \pm0.05$&  \cellcolor{llgray}0.15\\
Musk & $476$ & 0.43&0.91& $0.00$&\cellcolor{llgray}$0.21\pm 0.05$ & \cellcolor{llgray}$0.14 \pm 0.05$& \cellcolor{llgray}0.11\\
Breast Cancer &$569$ &0.34 &0.96& $0.00$& \cellcolor{llgray}$0.07\pm 0.02$ & \cellcolor{llgray}$0.07\pm 0.04$& \cellcolor{llgray}0.06\\
Pima Diabetes & $768$ & 0.35 & 0.75 & $0.00$& $0.07\pm 0.01$ & \cellcolor{llgray}$0.20\pm 0.10$& \cellcolor{llgray}0.14\\  
Spambase & $4,601$ & 0.39&0.94 & 3.47&$0.10\pm 0.02$&$0.06\pm 0.01$  & 0.03\\ 
Phoneme & $5,404$ & 0.29& 0.81 &  2.07 & $0.00\pm 0.00$ & $0.11\pm 0.01$ & 0.05\\ 
Mammography & $11,183$ & 0.02& 0.99 & 0.00&  $0.00\pm 0.00$ &  \cellcolor{llgray}$0.08\pm 0.07$&  \cellcolor{llgray}0.01\\
\bottomrule
\end{tabular}
\end{sc}
\end{small}
\end{center}
\end{table*}

As stated in Theorem \ref{thm:2}, the success of PIM depends on being fed by the outputs of a well-calibrated classifier \cite{guo_calibration,calib_reg,Krishnan}, so that these outputs approximate true probabilities. In these cases, PIM will give high-quality uncertainty estimates for the classification rates and derived quantities, given enough data. For the experiments that follow, a lower bound\cite{Kumar} $\Delta_{\tm{calib}}$ for the calibration error of the uncalibrated models is calculated.

To run the experiments, a classification model $\hat{f}$ with one hidden layer (having as many units as data features) and a sigmoid-activated output, is trained on $80\%$ of the data and evaluated on the remaining $20\%$, for nine UCI datasets. To quantify the spread of the uncertainty estimations, the train-test partition is randomly shuffled multiple times, forming an ensemble of 30 experiments. 

For each experiment, the distribution of ACC is sampled $B=20$ times by training $\hat{f}$ on $B$ transformations of the original training set, obtained by randomly sampling from it with replacement. Confidence intervals from the $B$ accuracies obtained in the validation ($=$test) sets are then computed by ranking their order statistics. The results from this Bootstrap (BS) method are compared with PIM and the binomial estimate in Table \ref{PIMvsBS}. As observed, PIM obtains confidence interval widths which are often narrower than BS (which ranks the order statistics of ACC) and therefore of higher quality, specially for small sample sizes. Yet, these are accurate enough\footnote{The effect of calibration on the spread of uncertainty estimations is investigated in section \ref{sec:classif_exp}} to overlap with the binomial estimate (except those few cases where $\hat{f}$ is miscalibrated). Benchmarking PIM against BS is important since the latter is a popular method \cite{ZhangThesis}, used for addressing uncertainties in many real-world applications, including time series forecasting \cite{BaggingTS}.

\subsubsection{Theoretical details}\label{sec:classif_theory}
The main theoretical result of section \ref{sec:PIM_classif}, namely Theorem \ref{thm:2}, is proved in this section. It will be understood that $\hat{Y}$ is a random variable taking the continuous values $\hat{y}=\hat{f}(x)\in[0,1]$. Using the classification threshold $\tau$ (by default $\tau=0.5$), $\hat{Y}$ is compared to the ground truth binary variable $Y$, taking values $y\in \mathcal{B}=\lbrace0,1\rbrace$, by applying the \emph{rounding} operation, $\rn{\hat{Y}}_\tau=\mathbbm{1}(\hat{Y}>\tau)$, which projects the values to the binary set $\mathcal{B}$.  

\textbf{Definition 1.} The index $l$ has been defined as a label for the set $\mathcal{K}=\lbrace \tm{TN}, \tm{TP}, \tm{FN}, \tm{FP}\rbrace$. This index can be written as the cartesian product $l=s_l\times v_l=\lbrace(s_l,v_l): s_l\in\lbrace\tm{T,\,F}\rbrace\;\tm{and}\;v_l\in\lbrace\tm{N,\,P}\rbrace\rbrace$. A mapping to the binary set $\mathcal{B}$ is introduced by putting a bar above the respective symbols according to:
\begin{equation}
 \bar{s}_l=\bar{v}_l\in\mathcal{B}\;\;\;\tm{for}\;\;\;s_l=\tm{T},\;\;\tm{and}\;\;\bar{s}_l=1-\bar{v}_l\in\mathcal{B}\;\;\;\tm{for}\;\;\;s_l=\tm{F}.
\end{equation}
In this way, a value of $l$ can be uniquely mapped to a pair of binary symbols $(\bar{s}_l,\bar{v}_l)$, taking on values in $\mathcal{B}$. Just as the rounding operator $\rn{\cdot}_\tau\equiv\rn{\cdot}_v$ projects to the set $\{\tm{N, P}\}$ of possibles values of $v_l$, we define $\rn{\cdot}_s$ as an operator projecting to the set $\{\tm{T, F}\}$ of possible values of $s_l$. Unless otherwise stated, the subscript $\tau$ may be dropped, for simplicity, from the rounding operator $\rn{\cdot}_\tau$.

\textbf{Definition 2.} The symbol $|a|$ is used to count the total number of elements in the set labeled by $a$; for instance, the quantity $|v_l|\in\lbrace|\tm{N}|, |\tm{P}|\rbrace$ takes on values denoting the total number of negatives or positives in the validation set. Using this notation, the classification rates $R_l=L_l/V_l$ can be written as as quotient of random variables $L_l$ and $V_l$ taking on the values, $|l|$ and $|v_l|$, respectively. 

\ \vspace{0.1cm}

\begin{lemma}\label{normFRk}
 The classification rates $R_l=L_l/V_l$ are asymptotically normal distributed with mean $\Pr(\rn{\hat{Y}}=\bar{s}_l\,|\,Y=\bar{v}_l)$ and variance  $|v_l|^{-1}\Pr(\rn{\hat{Y}}=\bar{s}_l\,|\,Y=\bar{v}_l)\,\Pr(\rn{\hat{Y}}=1-\bar{s}_l\,|\,Y=\bar{v}_l)$ in the limit when $|v_l|\rightarrow\infty$.
\end{lemma}

\begin{proof}
 Since $\rn{\hat{Y}}$ is a binary random variable, $L_l$ is Binomially distributed, so the result immediately follows after applying the de Moivre-Laplace central limit theorem.
\end{proof}
\begin{corollary}
 The accuracy of a binary classification algorithm is asymptotically normal distributed with mean $\mu_{\tm{ACC}} = p_{\tm{N}}\,\mu_{R_{\tm{TN}}}+p_{\tm{P}}\,\mu_{R_{\tm{TP}}}$ and variance $\sigma_{\tm{ACC}}^{2}=p_{\tm{N}}^{2}\,\sigma_{R_{\tm{TN}}}^{2}+p_{\tm{P}}^{2}\,\sigma_{R_{\tm{TP}}}^{2}$, with $\sigma_{\tm{ACC}}^{2}\rightarrow\mu_{\tm{ACC}}(1-\mu_{\tm{ACC}})/m$ as $m=|\tm{P}|+|\tm{N}|\rightarrow\infty$.
\end{corollary}
\begin{proof}
 This follows by writing the accuracy as a weighted sum $\tm{ACC}=p_{N}\,R_{\tm{TN}}+p_{P}\,R_{\tm{TP}}$ of independent and asymptotically normal random variables. As a consequence, the accuracy is also asymptotically normal distributed with mean $\mu_{\tm{ACC}} = p_{\tm{N}}\,\mu_{R_{\tm{TN}}}+p_{\tm{P}}\,\mu_{R_{\tm{TP}}}$. The independence of $R_{\tm{TN}}$ and $R_{\tm{TP}}$ (they refer to mutually exclusive subspaces) then implies that the variance is $\sigma_{\tm{ACC}}^{2}=p_{\tm{N}}^2\,\sigma_{R_{\tm{TN}}}^{2}+p_{\tm{P}}^2\,\sigma_{R_{\tm{TP}}}^{2}$. From Lemma \ref{normFRk}, $\sigma_{R_{\tm{TN}}}^{2}=\mu_{R_{\tm{TN}}}(1-\mu_{R_{\tm{TN}}})/|\tm{N}|$ and $\sigma_{R_{\tm{TP}}}^{2}=\mu_{R_{\tm{TP}}}(1-\mu_{R_{\tm{TP}}})/|\tm{P}|$. Therefore, $\sigma_{\tm{ACC}}^{2}$ differs from $\mu_{\tm{ACC}}(1-\mu_{\tm{ACC}})/m$ by a quantity of $O(1/m)$, with $m=|\tm{P}|+|\tm{N}|$ and $p_{\tm{N}}=|\tm{N}|/m$, $p_{\tm{P}}=|\tm{P}|/m$. This result was used in \eqref{dacc_n} to express the binomial confidence interval radius as $\delta(\tm{ACC})_{\mathcal{N}}=z_{p}\,\sigma_{\tm{ACC}}$.
\end{proof}

\begin{theorem}\label{thm:2}
If the output $\hat{y}\in[0,1]$ of a binary classifier is perfectly calibrated, \cite{guo_calibration} i.e. $\Pr(\rn{\hat{Y}} = Y\,|\, \hat{Y} = q) = q\;$ for all $q\in[0,1]$, then the quantiles $\hat{r}_{p}^{l}$ directly estimated by PIM from the validation errors $\varepsilon_l = y_l-\hat{y}_l$ are asymptotically consistent with the $p$-quantiles of the asymptotically normal distribution of the classification rates $R_l$.
\end{theorem}

Note that perfect calibration is impossible in all practical settings. However, there are empirical approximations (calibration methods), some of them used in section \ref{sec:classif_exp} below, which capture the essence of perfect calibration. A ``well-calibrated'' $\hat{f}$ is understood here as model that, by designed, is calibrated or that has been calibrated properly after applying a calibration method. Before proceeding with the proof of Theorem \ref{thm:2}, it helps to first visualize the meaning of the statement. In Figure \ref{adult_scaling}, a plot of the empirical distribution of $\hat{Y}\,|\,Y$ is shown for a neural network with one hidden layer predicting on the test set of the \emph{Adult} dataset in the UCI repository. It is noticed that, after applying a calibration method, the false predictions tend to cluster around the threshold $\tau=0.5$, following a kind of Gaussian-like envelope. For true predictions, these cluster around the ground truth but displaying long tails depending on the calibration method. PIM is applied to find quantiles for the errors around the targets (for false predictions, the target is $\tau$). The statement is then that, under certain conditions, these quantiles coincide with those of the distribution of the classification rates $R_l$ of the corresponding $\rn{\hat{Y}}$. 

\begin{figure*}[t]
\vskip 0.2in
\begin{center}
\centerline{\includegraphics[width=\columnwidth]{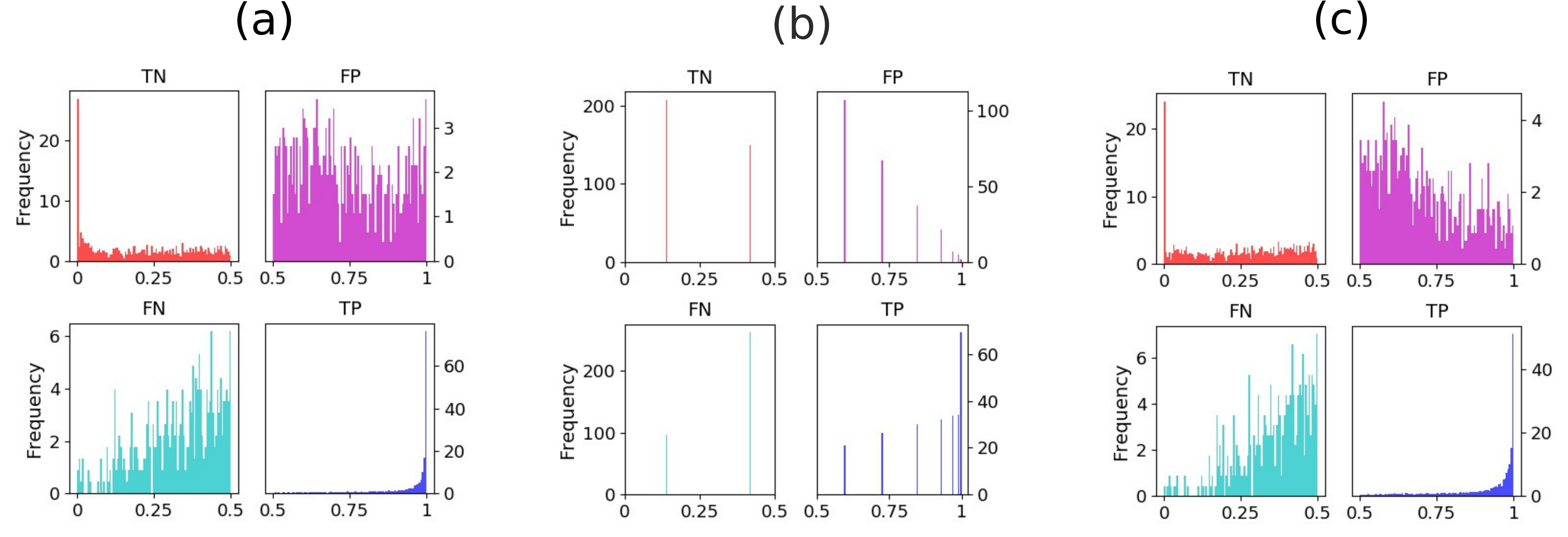}}
%\centerline{\includegraphics[scale=0.4]{scaling.eps}}
\caption{Confusion matrix shown as a normalized \emph{distribution} of probability scores $\bm{\hat{y}}$ thrown by a classifier on the default test set of the \emph{Adult} dataset: (a) original distributions (b) distributions after applying the Scaling-Binning calibrator \cite{Kumar} (c) distributions after applying the Platt calibrator.}
\label{adult_scaling}
\end{center}
\vskip -0.2in
\end{figure*}

\begin{proof}
The proof proceeds by first showing that, if a classifier is perfectly calibrated, the quantiles of the predicted targets $\hat{Y}$ are intimately connected with the expected value of the predictions. This is then used to pivot the errors $\varepsilon_l=y_l-\hat{y}_l$ with respect to the threshold $\tau$ when $s_l=\tm{F}$ (i.e. $y_l=\tau$) and with respect to the ground truth when $s_l=\tm{T}$ (i.e. $y_l=y$). 

Using the notation of Definition 1, we seek an identity which links positive and negative predictions with true and false predictions. This is
\begin{equation}
 \rn{\hat{Y}}_\tau := \mathbbm{1}(\hat{Y}> \tau)=\;\;\mathbbm{1}(Y=1)\,\mathbbm{1}(\rn{\hat{Y}}_s=\tm{T})+\mathbbm{1}(Y=0)\,\mathbbm{1}(\rn{\hat{Y}}_s=\tm{F}),
\end{equation}
which is valid for any threshold $\tau\in(0,1)$. Taking expectation value on both sides,
\begin{equation}\label{eq:phY}
 \begin{split}
  \Pr(\hat{Y}>\tau) &= \Pr(Y=1, \rn{\hat{Y}}_s=\tm{T})+\Pr(Y=0, \rn{\hat{Y}}_s=\tm{F})\\
  &= \Pr(\rn{\hat{Y}}_\tau=Y)
  = \int_0^1 \Pr(\rn{\hat{Y}}_\tau=Y\,|\,\hat{Y}=q_\tau(z))\,\rho_{\hat{Y}}(z)dz,
 \end{split}
\end{equation}
where $\rho_{\hat{Y}}$ is the probability density function of $\hat{Y}$ associated to the cumulative distribution function $F_{\hat{Y}}(\hat{y})=\int_0^{\hat{y}}\rho_{\hat{Y}}(\hat{y}')\,d\hat{y}'$, and $q_\tau$ belongs to a family of smooth functions $q_\tau:[0,1]\rightarrow[0,1]$ labeled by $\tau$. If the classifier is perfectly calibrated,
\begin{equation}
 \Pr(\rn{\hat{Y}}_\tau=Y\,|\,\hat{Y}=q_\tau(z))=q_\tau(z),\;\;\;\;\;\forall q_\tau(z)\in[0,1].
\end{equation}
Replacing this in \eqref{eq:phY} and denoting by $\langle q_\tau\rangle=\mathbb{E}_{\hat{Y}}(q_\tau)$ the expectation value of $q_\tau$, we obtain
\begin{equation}\label{eq:qFh}
 F_{\hat{Y}}(\tau)=\Pr(\hat{Y}\le\tau)=1-\langle q_\tau\rangle.
\end{equation}
Clearly, $\langle q_\tau\rangle\in(0,1)$, so \eqref{eq:qFh} states that $\tau$ coincides with the $(1-\langle q_\tau\rangle)$-quantile of $F_{\hat{Y}}$. For well-balanced datasets $p_\tm{N}\simeq p_\tm{P}\simeq1/2$, the classifier will presumably learn to predict aproximately the same amount of positive and negative predictions, so $\tau=1/2$ will coincide with the median ($=\,$mean) of $\hat{Y}$, which is a special case of \eqref{eq:qFh} for $\tau=\langle q_\tau\rangle=1/2$

As in section \ref{sec:condq}, we are interested in the errors commited by the predictive model. In that section, these were measured with respect to the median of $\hat{Y}$ (expected to coincide with $Y$). However, in the binary classification context, the median of $\hat{Y}$ is not necessarily close to $Y$, as Fig. \ref{adult_scaling}(c) suggests. The result implied by \eqref{eq:qFh} then suggests that the errors $\varepsilon_l=y_l-\hat{y}_l$ committed by the predictive model be measured relative to the ground truth $Y$ for $s_l=\tm{T}$ and relative to $\tau$ for $s_l=\tm{F}$. This leads to the quantile estimation problem for the error variable $E_l$ as finding the $\hat{r}_p^l$ such that the empirical error distribution functions evaluate to the confidence level $p$
\begin{equation}
 F_l(\hat{r}_p^l) := \dfrac{1}{|v_l|}\sum_{i:\, \rn{\hat{y}_i}=\bar{s}_l}\mathbbm{1}(|\varepsilon_l^i|\le \hat{r}_p^l) = p.
\end{equation}
Such quantiles are estimated by PIM as an alternative to calculating the quantiles of $\hat{Y}$ directly (similar to what was done in section \ref{sec:condq}). It is known that for a perfectly calibrated classifier, the accuracy is locally distributed as the average confidence\cite{guo_calibration} (here $\hat{Y}$). Since the classification rate $R_l$ is the accuracy in the subspace indexed by $l$, this shows that the quantiles of $R_l$ coincide with the corresponding quantiles of $\hat{Y}_l$. Furthermore, by Lemma \ref{normFRk}, the $R_l$ are asymptotically normally distributed.

\end{proof}

\subsubsection{Uncertainty propagation}
PIM uses four neurons $u_l$ to measure confidence intervals $\hat{r}_p^l:=\delta_p (R_l)$ for the classification rates $R_l$. Given the nature of classification $\delta_p (R_l)<1$ with probability one. Knowing the value of $p$ from the context, we can omit it from the $\delta$ subscript. It is then convenient to think of $\delta_p (R_l):=\delta R_l$ as a uncertainty that can be propagated to quantities dependent on $\{R_l\}$ using Taylor's theorem. This was done in \eqref{dacc} to go from $\tm{ACC}=p_{\tm{N}}R_{\tm{TN}}+p_{\tm{P}}R_{\tm{TP}}$ to $\delta\tm{ACC}=p_{\tm{N}}\delta R_{\tm{TN}}+p_{\tm{P}}\delta R_{\tm{TP}}$ by $\delta$-differentiating both sides. The result is straightforward in this case because the relationship connecting the classification rates with the quantity of interest is linear. No error in the Taylor expansion is committed in this case. In this section, we would like to consider non-linear relationships and use uncertainty propagation techniques --- as in the natural sciences --- to find the associated uncertainties. 

With $\sim$ denoting the asymptotic value around which the classification rates cluster, it has been shown in Lemma \ref{normFRk} that
\begin{equation}
\begin{split}
 R_{\tm{TP}}&=\dfrac{|\tm{TP}|}{|\tm{TP}|+|\tm{FN}|}\sim\Pr(\rn{\hat{Y}}=1\,|\,Y=1),\\
 R_{\tm{FN}}&=1-R_{\tm{TP}} \,\,\sim \Pr(\rn{\hat{Y}}=0\,|\,Y=1),\\
 R_{\tm{FP}}&=\dfrac{|\tm{FP}|}{|\tm{FP}|+|\tm{TN}|}\sim \Pr(\rn{\hat{Y}}=1\,|\,Y=0),\\
 R_{\tm{TN}}&=1-R_{\tm{FP}} \,\,\sim \Pr(\rn{\hat{Y}}=0\,|\,Y=0).
 \end{split}
 \label{crates}	
\end{equation}
 It is of interest to estimate the uncertainty of other important rates, namely, positive predictive value ($R_{\tm{TP}}^{*}$, a.k.a. precision), the false discovery rate ($R_{\tm{FN}}^{*}$), the negative predictive value ($R_{\tm{TN}}^{*}$), and the false omission rate ($R_{\tm{FP}}^{*}$). These are obtained after interchanging the roles of predictions and ground truths. By symmetry,
\begin{equation}
\begin{split}
 R_{\tm{TP}}^{*}&=\dfrac{|\tm{TP}|}{|\tm{TP}|+|\tm{FP}|}\sim\Pr(Y=1\,|\,\rn{\hat{Y}}=1),\\
 R_{\tm{FN}}^{*}&=1-R_{\tm{TP}}^{*} \,\,\sim\Pr(Y=0\,|\,\rn{\hat{Y}}=1),\\
 R_{\tm{TN}}^{*}&=\dfrac{|\tm{TN}|}{|\tm{TN}|+|\tm{FN}|}\sim\Pr(Y=0\,|\,\rn{\hat{Y}}=0),\\
 R_{\tm{FP}}^{*}&=1-R_{\tm{TN}}^{*} \,\,\sim\Pr(Y=1\,|\,\rn{\hat{Y}}=0).
 \end{split}
 \label{posteriors}
\end{equation}
Taking the $\delta R_l$ learned by PIM as independent variables (the neurons $u_l$ are independent), it is assumed that the uncertainties of any smooth function $g$ of $\lbrace R_{l}\rbrace$ can be approximated by Taylor's expansion:
\begin{equation}
 \delta g(\lbrace R_{l}\rbrace) = \sum_{q}\Bigl|\dfrac{\partial g}{\partial R_q}\Bigr|\,\delta R_q + \tfrac{1}{2}\sum_{q\in\tm{F}}\Bigl|\dfrac{\partial^2 g}{\partial R_q^2}\Bigr|\,(\delta R_q)^2+\cdots.
 \label{deltaf}
\end{equation}
When an \emph{approximation} of $\delta g$ is enough, only second order corrections are taken into account for false predictions, assuming they are more uncertain due to the (good enough) classification algorithm commiting them less frequently. Now, by using Bayes' theorem, the asymptotic values in \eqref{crates} and \eqref{posteriors} can be connected as
\begin{equation}
\begin{split}
 \Pr(Y=\bar{v}_l\,|\,\rn{\hat{Y}}&=\bar{s}_l)=\dfrac{\Pr(\rn{\hat{Y}}=\bar{s}_l\,|\,Y=\bar{v}_l)\Pr(Y=\bar{v}_l)}{\Pr(\rn{\hat{Y}}=\bar{s}_l)},\\
 &=\dfrac{\Pr(\rn{\hat{Y}}=\bar{s}_l\,|\,Y=\bar{v}_l)\Pr(Y=\bar{v}_l)}{\Pr(\rn{\hat{Y}}=\bar{s}_l\,|\,Y=0)\Pr(Y=0)+\Pr(\rn{\hat{Y}}=\bar{s}_l\,|\,Y=1)\Pr(Y=1)}.
 \end{split}
\end{equation}
From this, it is easy to see that the most probable values of $R_l^*$ and $R_l$ are simply related. For instance,
\begin{equation}
\begin{split}
 R_{\tm{TP}}^{*} &\sim \dfrac{R_{\tm{TP}}\,p_{\tm{P}}}{R_{\tm{TP}}\,p_{\tm{P}}+R_{\tm{FP}}\,p_{\tm{N}}},\\
 R_{\tm{TN}}^{*} &\sim \dfrac{R_{\tm{TN}}\,p_{\tm{N}}}{R_{\tm{FN}}\,p_{\tm{P}}+R_{\tm{TN}}\,p_{\tm{N}}}.
\end{split}
\end{equation}
This relationships are examples of the $g$ function above, so by \eqref{deltaf},
the uncertainties are related as
\begin{equation}
 \begin{split}
  \delta R_{\tm{TP}}^{*}&\sim\dfrac{p_\tm{N}}{p_\tm{P}}\,R_{\tm{TP}}^{*2}\Biggl[\dfrac{R_{\tm{FP}}}{R_{\tm{TP}}}\,\dfrac{\delta R_{\tm{TP}}}{R_{\tm{TP}}}+\dfrac{\delta R_{\tm{FP}}}{R_{\tm{TP}}}
 +\dfrac{p_\tm{N}}{p_\tm{P}}\,R_{\tm{TP}}^{*}\left(\dfrac{\delta R_{\tm{FP}}}{R_{\tm{TP}}}\right)^2\Biggr],\\
 \delta R_{\tm{TN}}^{*}&\sim\dfrac{p_\tm{P}}{p_\tm{N}}\,R_{\tm{TN}}^{*2}\Biggl[\dfrac{R_{\tm{FN}}}{R_{\tm{TN}}}\,\dfrac{\delta R_{\tm{TN}}}{R_{\tm{TN}}}+\dfrac{\delta R_{\tm{FN}}}{R_{\tm{TN}}}
 +\dfrac{p_\tm{P}}{p_\tm{N}}\,R_{\tm{TN}}^{*}\left(\dfrac{\delta R_{\tm{FN}}}{R_{\tm{TN}}}\right)^2\Biggr].
 \end{split}
\end{equation}
Uncertainties for other metrics derived from $R_l$, e.g the F1 score, can be obtained in a similar manner.

\subsubsection{Effect of calibration}\label{sec:classif_exp}
In the experiments of section \ref{sec:bin_classif}, no hyperparameter optimization is done. The only thing that is varied is the activation of the hidden layer (relu and tanh), and the calibration method for the predictions. Best results are reported.

It is noticed during experimentation that sometimes some of the estimates $\hat{r}_{p}^l$ stay very close to their initial values (within a tolerance of $10^{-7}$), for which NA is used when requesting their optimal values. This is often due to scarcity of data, since $\hat{r}_{p}^l$ may not be updated for each of FN, FP, TN, TP within a mini-batch. For a classifier with relatively few false predictions, for instance, the optimal $\hat{r}_{p}^F$ may not be found. Therefore, uncertainty on the rates of false predictions cannot be currently evaluated for small datasets.

This problem is not found for datasets like the \emph{Adult} dataset ($48,842$ samples). Uncertainties estimated by PIM and propagated according to the technique described above are shown in  Table \ref{PIMcrates}. It is observed that calibrating the predictions leads most of the time to a decrease in the magnitude of the uncertainties, with a stabilization of the corresponding variance. However, for some rates the magnitude of the uncertainties still look too conservative. This is presumably due to the calibration method based on scaling not reaching calibrated-enough predictions. \cite{Kumar} The Scaling-Binary calibrator \cite{Kumar} (whose effects are shown in Figure \ref{adult_scaling}) is not considered in Table \ref{PIMcrates} since it does not give a continuous distribution of predictions, as required by PIM. Further research is desirable, combining the idea behind PIM with a suitable calibration method into one framework.

\setlength{\tabcolsep}{4pt}
\begin{table*}[hbt]
\caption{Classification metrics with uncertainties learned by PIM in the Adult dataset: median $\pm$ median of uncertainty (median absolute deviation of uncertainty) over an ensemble of 20 experiments. These are calculated with and without calibration. Temperature scaling is done using an external library. \cite{temp_scale}}
\label{PIMcrates}
\vskip 0.15in
\begin{center}
\begin{small}
\begin{sc}
\begin{tabular}{ccccc}
\toprule
Calibration & $R_{\tm{FP}}$ & $R_{\tm{TP}}^{*}$ & ACC &\\
\midrule
Uncalibrated &  $0.38\pm0.25\,(0.03)$& $0.89\pm0.09\,(0.01)$ & $0.85\pm0.20\,(0.02)$   \\
Platt scaling&  $0.41\pm0.17\,(0.02)$& $0.88\pm0.07\,(0.01)$ & $0.85\pm0.20\,(0.00)$ &  \\
Temperature scaling & $0.38\pm0.18\,(0.01)$& $0.89\pm0.07\,(0.00)$ & $0.85\pm0.21\,(0.01)$ &  \\
\bottomrule
\end{tabular}
\end{sc}
\end{small}
\end{center}
\vskip -0.1in
\end{table*}

\end{document}